\documentclass[twoside,11pt]{article}

\usepackage{microtype}
\usepackage{graphicx}
\usepackage{subcaption}

\usepackage{booktabs} 

\usepackage[preprint]{jmlr2e}

\usepackage{amsmath}
\usepackage{amssymb}
\usepackage{mathtools}

\usepackage[capitalize,noabbrev]{cleveref}

\begin{document}

\title{Loss Functions for Classification using Structured Entropy}

\author{\name Brian Lucena \email brian@numeristical.com \\
       \addr Numeristical\\
       Alameda, CA 94501, USA
       }
\ShortHeadings{Structured Entropy Loss Functions}{Brian Lucena}      
\maketitle

\begin{abstract}
Cross-entropy loss is the standard metric used to train classification models in deep learning and gradient boosting.  It is well-known that this loss function fails to account for similarities between the different values of the target.  We propose a generalization of entropy called {\em structured entropy} which uses a random partition to incorporate the structure of the target variable in a manner which retains many theoretical properties of standard entropy.  We show that a structured cross-entropy loss yields better results on several classification problems where the target variable has an a priori known structure.  The approach is simple, flexible, easily computable, and does not rely on a hierarchically defined notion of structure.\end{abstract}

\section{Introduction}
The cross-entropy loss function (e.g. \cite{lecun2015deep}) is the standard optimization metric for training neural networks and gradient boosting models.  However, cross-entropy does not consider any notion of similarity or structure of the possible values of the variable. Intuitively, when the correct answer is ``maple",  a model which predicts ``oak" is less wrong than one that predicts ``trout".  This weakness is well-known (e.g. \cite{wu2019hierarchical}, \cite{Bertinetto_2020_CVPR}) and there have been numerous attempts to address it.

From one perspective, the inability of cross-entropy to make distinctions between more and less similar target values stems from the definition of entropy of a random variable (\cite{Shannon48}).  Since entropy treats all possible values of the random variable as equally different, that property flows down to KL-divergence and cross-entropy.  Therefore, it would be desirable to have a version of entropy that recognizes the structure of the values of the target random variable.  Ideally, we would wish to have a variant that also satisfies the following properties:
\begin{enumerate}
\item{It retains the fundamental properties of entropy and its derived quantities such as conditional entropy, cross-entropy and KL-divergence.}
\item{It is flexible and simple to define {\em any} kind structure based on real-world knowledge (i.e. it is not restricted to hierarchical structure).}
\item{It is efficient to repeatedly compute the derived cross-entropy (e.g. as a loss function in neural networks or gradient boosting).}
\end{enumerate}

\subsection{Contributions}
This paper makes the following contributions.  We define a notion of {\em structured entropy} wherein the values of the random variable are equipped with a notion of similarity.  Specifically, we define structure by means of a {\em random partition} which effectively gives weights to different partitions of the state space.  These partitions represent meaningful ``coarsenings", and thus the structured entropy can be seen as a weighted average of standard entropies taken over different degrees of coarse/fine structure.  By using partitions, rather than notions specific to hierarchical structure, our space of loss functions has much higher dimensionality and is therefore broader and more generally applicable.  We show that structured entropy satisfies the three criteria mentioned in the introduction: theoretical soundness, flexibility, and efficient computability.  We further propose an approach using a {\em variable} loss function in iterative methods, which requires only a graphical representation of the structure of the outcome variable. We demonstrate improved performance on a variety of prediction problems.

The rest of the paper is organized as follows.  Section~\ref{sec:related-work} discusses related work and positions our contributions accordingly.  Section~\ref{sec:defentr} defines structured entropy and shows that it is a natural extension of entropy, sharing many theoretical properties and derived variants.  Section~\ref{sec:howtostructure} discusses how to define random partitions in practice, illustrating several examples.  Section~\ref{sec:experiments} shows the improvement yielded by this approach on several prediction problems involving various kinds of structure on the target variable.  

\section{Related Work}
\label{sec:related-work}
There are several bodies of work that relate to the ideas presented here. A large body of work is focused generally on exploiting hierarchy in classification problems.  \cite{Chen2018UnderstandingTI} demonstrates that training on fine-grained labels (using the standard cross-entropy loss function) can improve classification performance on the coarse-grained labels, specifically for CNNs on image classification problems.  The message of our paper is consistent with this, but in fact more general, indicating that useful information exists at {\em all} levels of the hierarchy, and can be exploited by structured cross-entropy.   A related work (\cite{MoHierarchical}) finds a similar benefit to fine-grained labels in the context of text classification and advocates for ``active over-labeling" as a route to improved performance.  \cite{Bertinetto_2020_CVPR} define approaches using {\em hierarchical cross-entropy}  and {\em soft-labels} to encode hierarchical information and demonstrate an ability to trade-off ``more mistakes" with ``better mistakes" in a manner that subsumes the prior art.  \cite{brust2019integrating} shows accuracy improvements by exploiting DAG-based hierarchicies.  A hierarchical alternative to cross-entropy is defined in \cite{wu2019hierarchical}  but the paper casts doubt upon its effectiveness as a loss function for training.  \cite{hoyoux2016can} proposes variations to k-nearest neighbors and support vector machines that attempt to utilize hierarchical structure but finds little to no improvement over the ``flat" alternatives.  In \cite{zhao2017hierarchical}, the hierarchy is used to find different features for classifying different nodes in the hierarchical tree. In \cite{kosmopoulos2015probabilistic} a cascading approach is explored, training different classifiers at each node of the tree, but combining them in a non-greedy manner. \cite{zhao2017hierarchical} proposes an alternative to cross-entropy specifically to combat noisy labels, though it is not hierarchical in nature. \cite{wang2017local} uses the hierarchy to make coarser predictions when confidence is lower about the more fine-grained label. In \cite{ristin2015categories}, a variant of random forests called nearest class mean (NCM) forests are further modified to use both coarse and fine grained labels.  Earlier work on hierarchical classification includes \cite{NIPS2011_d5cfead9} and a survey can be found in \cite{silla2011survey}.   Our work differs from all of the above in that we define structure via partitions.  As such, our approach is more generally applicable to different kinds of structure, and does not rely on hierarchical notions.

Another body of work attempts to generalize entropy to reflect structure, but without necessarily considering its potential use as a replacement for cross-entropy loss.  A similar attempt to generalize entropy is found in \cite{pmlr-v108-posada20a}, which introduced to the machine learning community a notion of entropy first defined by \cite{leinster2012measuring}.  This definition takes a similarity matrix between the values of the variable and uses it to calculate ``geometry-aware" versions of entropy and related concepts.  They apply this notion to computing image barycenters and training generative models, among other applications, but do not use it as a loss function, perhaps because its computational complexity remains high.  An earlier notion, {\em weighted entropy} (\cite{BelisGuiasu1}, \cite{Guiasu1977InformationTW}, for a modern treatment see \cite{pocock2012feature}) assigns a weight to each distinct value of the variable, but does not convey similarity between values.  Notably, the derived weighted mutual information can be negative (\cite{pocock2012feature}, \cite{e19110631}).  Our notion of structured entropy differs from both of these.

The literature on ``optimal transport" (\cite{villani2008optimal}) or ``earth-movers" distance  (\cite{peyre2019computational}) is related as it provides an alternative to KL-divergence.  Work in this direction includes \cite{pmlr-v84-genevay18a}, \cite{pmlr-v32-cuturi14}, and \cite{WassLossFrogner}.  However, these metrics are quite computationally expensive and therefore not suitable for use as a loss function in an iterative algorithm, even when approximated as in \cite{cuturi-lightspeed}.
  
The approach of ``contrastive learning"  (e.g. \cite{khosla2020supervised}, \cite{NIPS2016-6b180037}, \cite{tian2019contrastive}, \cite{oord2018representation}) attempts to learn structural representations of the data, but by processing the data itself, rather than exploiting known (a priori) structure in the labels.

Finally, previous work by this author (\cite{pmlr-v108-lucena20a}) defined the notion of a {\em terrain} from which meaningful partitions could be examined to exploit structure in categorical variables.  However, that notion applied only to predictor variables, not target variables, and its use relied on tree-based methods.

\section{Structured Entropy}
\label{sec:defentr}

Here we define structured entropy and prove many of its properties.  The ``nested" nature of the definitions (involving sets of sets of sets) can be confusing,  so we will present a simple example and use it to illustrate as we proceed through the material. For a review of standard entropy, we refer the reader to \cite{cover1999elements}.

We will use the following conventions:  lowercase letters ($a$) represent basic elements, capital letters ($A$) represent sets of elements, calligraphic letters ($\mathcal{A}$) represent sets of sets of elements (e.g. a partition), and block letters ($\mathbb{A}$) represent sets of partitions (i.e. sets of sets of sets of elements).   We will use the notation $S_Y$ to represent the set of possible values of a random variable $Y$.  For simplicity, we will assume that all variables have a finite state space.

As a running example, let $Y$ represent the outcome of hospital patients 30 days post-discharge and let $S_Y=\{1,2,3\}$ where 1 represents that the patient died (without ever returning to the hospital),  2 represents readmission to the hospital within 30 days, and 3 represents that neither of the other events occurred.  We intend to predict $Y$ using a predictor $X$ which has information about the vital signs, lab values, and medical history of the patient.  We expect the patients who die and the patients who readmit to look more similar to each other (in terms of their $X$ values) than either group will be to the uneventful patients.  We want to represent this prior belief in our calculations of entropy and its derived quantities.  Let $p_1, p_2, p_3$ represent the probabilities that $Y$=1,2,3 respectively (where $p_3 = 1-p_1-p_2$).

\subsection{Definitions}

\begin{definition}
A partition $\mathcal{P}$ of a set $S$ is a set of non-empty subsets of $S$ such that for every element $y \in S$, there exists exactly one element of $\mathcal{P}$ which contains $y$.  The elements of $\mathcal{P}$ are referred to as {\em blocks} of the partition.
\end{definition}

\begin{definition}
The {\em singleton partition} of a set $S$ the partition $\mathcal{S}$ given by $\{\{s\}: s \in S\}$.
\end{definition}

With a slight abuse of notation, we will define the coarsening of a variable $Y$ with respect to a partition $\mathcal{S}$ as follows:

\begin{definition}
Given a random variable $Y$ and a partition $\mathcal{S}$ on $S_Y$, we define $\mathcal{S}(Y)$ to be a random variable such that $\mathcal{S}(Y) = B$ if and only iff $Y=y$ and $y \in B$ (where $B$ is a block in $\mathcal{S}$) .
\end{definition}

In other words, $\mathcal{S}(Y)$ is a function of $Y$ which maps the value of $Y$ to the subset in the partition $\mathcal{S}$ which contains it.  So, $\mathcal{S}(Y)$ is a random variable, but merely a function of $Y$.

\begin{remark}
Since $\mathcal{S}(Y)$ is a function of $Y$, we can conclude that $H(\mathcal{S}(Y)) \leq H(Y)$ (where $H(Y)$ denotes the Shannon entropy of a random variable $Y$).
\end{remark}

\begin{definition} 
A {\em structure} $\mathbb{S}$ of a set $S$ is a set of partitions of $S$.  We also refer to a structure of a random variable $Y$ as a set of partitions of $S_Y$.
\end{definition}

\begin{definition}
A {\em random partition} of a set $S$ is a random variable $Z$ which takes values on a structure $\mathbb{S}$ of $Y$.
\end{definition}

\begin{definition}
A {\em structured random variable} is a random variable $Y$ paired with a random partition $Z$ of the set $S_Y$. We denote the structured random variable by $Y_Z$.
\end{definition}

Returning to our example, let $\mathcal{S}_1$ be the singleton partition, i.e. $\mathcal{S}_1 = \{\{1\},\{2\},\{3\}\}$, let $\mathcal{S}_2$ be the partition $ \{\{1,2\},\{3\}\}$, and let $\mathbb{S}=\{\mathcal{S}_1,\mathcal{S}_2\}$ be a structure of $Y$.  This choice of structure represents our prior belief that 1 and 2 are more similar to one another.  When we calculate entropy and related quantities, we will essentially ``average" over two points of view: the first where 1, 2, and 3 are each distinct and separate, and a second where 1 and 2 are equivalent (representing ``sick" patients) and 3 is distinct.

\subsubsection{Structured Entropy}
\begin{definition}
Let $Y_Z$ be a structured random variable, where $Z$ takes values on a structure $\mathbb{S}$ of $Y$. The {\em structured entropy} of  $Y_Z$ is given by:
\begin{equation*}
H(Y_Z)=\sum_{\mathcal{S}_i \in \mathbb{S}} P(Z=\mathcal{S}_i) \sum_{S \in \mathcal{S}_i} -P(Y \in S) \log(P(Y\in S))
\end{equation*}
Put another way,
\begin{equation}
H(Y_Z)=\sum_{\mathcal{S}_i \in \mathbb{S}} P(Z=\mathcal{S}_i) H(\mathcal{S}_i(Y))
\end{equation}
\end{definition}

\begin{definition}
Given a structured random variable $Y_Z$, let the {\em random block} $Z(Y)$ be defined by $Z(Y) = \mathcal{S}_i(Y)$ when $Z=\mathcal{S}_i$.  In other words, $Z(Y)$ is a random variable (function of $Y$ and $Z$) taking values in $\mathcal{T}=\bigcup_{\mathcal{S}_i \in \mathbb{S}} \mathcal{S}_i$, which represents the unique subset of $S_Y$ that contains the sampled value of $Y$ and is in the sampled partition of $Z$.
\end{definition}

\begin{remark}
The structured entropy is equivalent to the conditional (standard) entropy of $Z(Y)$ given $Z$:
\begin{equation}
H(Y_Z)=H(Z(Y)|Z)
\end{equation}
\end{remark}

\begin{definition}
The {\em trivial random partition} of a set $S$ is given by $\mathbb{S}=\{\mathcal{S}\}$  where $\mathcal{S} = \{\{y\}: y \in S(Y)\}$ and $P(Z=\mathcal{S})=1$.  In other words, the structure contains only the partition containing all singletons, and that partition has probability one.
\end{definition}

\begin{remark}
Let $Z$ be the trivial random partition on $S_Y$.  Then $H(Y_Z) = H(Y)$.
\end{remark}

Continuing the example, let $Z$ be a random partition taking values on $\mathbb{S}$, where $q_1 = P(Z=\mathcal{S}_1)$ and $q_2 = 1-q_1 = P(Z=\mathcal{S}_2)$.  Here, $Z$ determines how much weight to give to the singleton partition $\mathcal{S}_1$ relative to $\mathcal{S}_2$.  When $q_1$ is near one, then $H(Y_Z)$ will be close to $H(Y)$, whereas when $q_1$ is near zero, $H(Y_Z)$ will be close to the standard entropy that we would have if we ``coarsened" our outcome and treated death and readmission as the same.  By choosing $0<q_1<1$ we strike a balance between these extremes.

Now consider the random block $Z(Y)$ which takes values in the set $\mathcal{T} = \bigcup_{\mathcal{S}_i \in \mathbb{S}} \mathcal{S}_i$.  In this example, $\mathcal{T} = \{\{1\},\{2\},\{3\},\{1,2\}\}$. These are all the different blocks used by the partitions in $\mathbb{S}$.   Recall that $Z(Y)$ is the element of $Z$ that contains the value $Y$.  In this case:
\begin{equation*}
Z(Y) = 
    \begin{cases}
      \{1\}, & \text{with probability }\ q_1 p_1\\
      \{2\}, & \text{with probability }\ q_1 p_2\\
      \{3\}, & \text{with probability }\ q_1 p_3 + q_2 p_3 = p_3\\
      \{1,2\}, & \text{with probability }\ q_2 (p_1+p_2)
    \end{cases}
\end{equation*}

We can think about the structured entropy $H(Y_Z)$ as a weighted combination of two standard entropy values.  The first value is $H(Y)$, which corresponds to the scenario where $Z=\mathcal{S}_1$, the singleton partition.  The second value is  $h(p_1+p_2)$ where $h$ is the binary entropy function.  This corresponds to the scenario where $Z = \mathcal{S}_2$.  In that scenario, we conglomerate outcomes 1 and 2 of $Y$, and so the resulting entropy is that of a binary variable.

Alternatively, we can consider the random block $Z(Y)$ and easily verify that $H(Y_Z) = H(Z(Y)|Z)$.  In other words, structured entropy is just the conditional entropy of the random block given the partition.  Note that, in general, if $Y$ has $n$ possible values , then $Z(Y)$ can take on $O(2^n)$ different values - one for each subset of the $n$ possible values.  This high dimensionality is key to the flexibility of structured entropy.

One might ask what distribution $(p_1, p_2, p_3)$ would maximize the structured entropy of $Y_Z$ (given $q_1$).  Looking at the two limiting cases, we see that when $q_1 = 1$ we only care about the singleton partition, and the maximum entropy should be achieved at $(1/3, 1/3, /1/3)$.  Meanwhile, as $q_1 \rightarrow 0$ we care only about the coarser partition, and a symmetry argument suggests that the maximum occurs at $(.25, .25, .5)$.  Writing out the structured entropy and using basic calculus, it is easy to show that the maximum entropy occurs at:
\begin{equation*}
p_1 = p_2 = \frac{1}{2(1+2^{-q_1})}
\end{equation*}

\subsection{Related Concepts and Properties}
Here we extend our definition of structured entropy to many related quantities.  The proofs of the theorems are straightforward and deferred to the Appendix. Unless otherwise specified, all definitions and theorems assume the {\em default setting} presented below:

Let $Y$ be a structured random variable with a random partition $Z$ defined on a structure $\mathbb{S}$ over $S(Y)$.  Let $X$ be a structured random variable with a random partition $W$ defined on a structure $\mathbb{R}$ over $S(X)$.  
Let $q_i = P(Z=\mathcal{S}_i)$ and let $r_i = P(W=\mathcal{R}_i)$.

\subsubsection{Conditional Structured Entropy}
\begin{definition}
\label{cond-str-ent}
The {\em conditional entropy} $H(Y_Z|X_W)$ is given by:
\begin{equation}
H(Y_Z|X_W) = \sum_{\mathcal{S}_i \in \mathbb{S}} \sum_{\mathcal{R}_j\in \mathbb{R}} q_i r_j H(\mathcal{S}_i(Y)|\mathcal{R}_j(X))
\end{equation}
\end{definition}

Since the structured entropy is just a convex combination of Shannon entropies, it is easy to show that conditioning can only reduce structured entropy.

\begin{theorem}
$H(Y_Z|X_W) \leq H(Y_Z)$
\end{theorem}

\subsubsection{Relative Structured Entropy}
\begin{definition}
Let $Y$ and $Y^{\prime}$ be two random variables over the same space $S_Y$ and let $Z$ be a random partition on $S_Y$.  Let $q_i = P(Z=\mathcal{S}_i)$.  The {\em structured relative entropy} $D(Y_Z||Y^{\prime}_Z)$ is given by:
\begin{equation*}
D_{\mathbb{S}}(Y_Z||Y^{\prime}_Z) = \sum_{\mathcal{S}_i \in \mathbb{S}}q_i \sum_{S \in \mathcal{S}_i}  P(Y \in S)\log\left(\frac{P(Y \in S)}{P(Y^{\prime} \in S)}\right)
\end{equation*}
or equivalently:
\begin{equation}
D_{\mathbb{S}}(Y_Z||Y^{\prime}_Z) = \sum_{\mathcal{S}_i \in \mathbb{S}}q_i D(\mathcal{S}_i(Y)||\mathcal{S}_i(Y^{\prime}))
\end{equation}
\end{definition}
In other words, we take the expectation (over the random partition $Z$) of the relative entropy values between the coarsened versions of $Y$ and $Y^{\prime}$.

\subsubsection{Structured Mutual Information}
We can define the mutual information for structured random variables in the same manner as the standard mutual information, and it obeys the same symmetries.

\begin{definition}
The {\em structured mutual information} between $Y_Z$ and $X_W$ is given by:
\begin{equation}
I(Y_Z;X_W) = H(Y_Z) - H(Y_Z|X_W)
\end{equation}
\end{definition}
The structured mutual information is symmetric.
\begin{theorem}
$I(Y_Z;X_W) = I(X_W;Y_Z)$
\end{theorem}

\subsubsection{Joint Structured Entropy}
The notion of joint entropy similarly generalizes to the structured case.  To make this rigorous, we must first develop notation and operations to create joint partitions and structures from those of each individual variable.

\begin{definition}
Let $A\subseteq S_X$ and $B \subseteq S_Y$.  Define the set $(A,B) \subseteq S_X \times S_Y$ by:
\begin{equation}
(A,B) = \{(a,b): a \in A \mbox{ and }b\in B\}
\end{equation}
\end{definition}

\begin{definition}
Let $\mathcal{R}$ be a partition of $S_X$ and $\mathcal{S}$ be a partition of $S_Y$.  Define  ($\mathcal{R}, \mathcal{S}$) to be a partition of $S_X \times S_Y$ where:
\begin{equation}
(\mathcal{R}, \mathcal{S}) = \{(A,B): A \in \mathcal{R}  \mbox{ and } B \in \mathcal{S}\}
\end{equation}
\end{definition}

It is straightforward to verify that $(\mathcal{R}, \mathcal{S})$ is indeed a partition of $S_X \times S_Y$.

\begin{definition}
Define  $\mathbb{R} \times \mathbb{S}$ to be a structure of $S_X \times S_Y$ where:
\begin{equation}
\mathbb{R} \times \mathbb{S} = \{(\mathcal{R}_i, \mathcal{S}_j): \mathcal{R}_i \in \mathbb{R}  \mbox{ and } \mathcal{S}_j \in \mathbb{S}\}
\end{equation}
\end{definition}

\begin{definition}
Define the {\em joint random partition} $Z \times W$ by:
\begin{equation}
W \times Z = (\mathcal{R}_i, \mathcal{S}_j) \mbox{ with probability } r_i q_j
\end{equation}
\end{definition}

\begin{definition}
Define the {\em joint structured entropy} of $X_W$ and $Y_Z$ to be:
\begin{equation}
H(X_W, Y_Z) = H(X,Y)_{W\times Z}
\end{equation}
\end{definition}
That is, we treat the pair $(X,Y)$ as a single structured random variable with random partition $W \times Z$.  The joint structured entropy can be written as a sum in the same way as the standard variant.

\begin{theorem}
$H(X_W, Y_Z) = H(X_W) + H(Y_Z|X_W)$
\end{theorem}

\subsection{Entropy and Loss Functions}
\label{sec:lossfn}
In supervised learning, the purpose of a loss function is to measure the quality of a model's predictions.  There are typically two places where a loss function is used:
\begin{enumerate}
\item In iterative methods, to evaluate a current solution, and guide the next step of the iteration (e.g. gradient descent).  In this case, the training data is used to compute the loss function.
\item After one or more models have been created, to assess their quality and compare their performance.  Here, an independent test set is typically used.
\end{enumerate}

How does entropy relate to loss functions?  Consider the following thought experiment, which gives one interpretation of how it arises.  

Assume we have training data from some probability distribution $P(X,Y) = P(X)P(Y|X)$ represented by a set of $n$ paired observations $\{(x_l, y_l)\}$ where $X$ represents the predictor variable(s) and $Y$ represents the target.  For simplicity, assume $S_X$ and $S_Y$ are finite sets: $S_X = \{1,\ldots,m\}$ and $S_Y=\{1,\ldots,k\}$.  We wish to build a model using the training data such that given a test value $x \in S_X$, the model will output a probability distribution on $S_Y$.  Put another way, we wish to estimate the distribution $P(Y|X)$.  We represent this model by another distribution $Q(Y|X)$.  Let $p_{i,j}=P(Y=j|X=i)$ and let $q_{i,j} = Q(Y=j|X=i)$. The $p_{i,j}$ are the ``true" conditional probabilities and the $q_{i,j}$ are the model's corresponding estimates.  We would like to have a metric that measures the quality of our set of  $\{q_{i,j}\}$ given the true values $\{p_{i,j}\}$.  Let $r_i = P(X=i)$.  The $r_i$ values will enter into the metric as well, since it may be more important to have good estimates for the values of $X$ that are more probable.

For each value of $X$ we want to compare the true distribution $P(Y|X=i)$ to our estimated distribution $Q(Y|X=i)$.  The standard way to do this is to use KL-Divergence:
\begin{equation*}
D(P(Y|X=i)||Q(Y|X=i)) = \sum_{j \in S_Y} p_{i,j} \log \left(\frac{p_{i,j}}{q_{i,j}}\right)
\end{equation*}
Averaged across the different values of $X$ (and weighted by their respective probabilities) we get:
\begin{equation*}
\label{kldiv-obj}
D(P(Y|X)||Q(Y|X)) = \sum_{i\in S_X} r_i \sum_{j \in S_Y} p_{i,j} \log \left(\frac{p_{i,j}}{q_{i,j}}\right)
\end{equation*}
Note that the KL-divergence, conditioned in this way, also depends on the marginal distribution $P(X)$.

Simple algebra yields the following result:
\begin{equation*}
D(P(Y|X)||Q(Y|X)) = H^{\dag}(P(Y|X),Q(Y|X)) - H(Y|X)
\end{equation*}
where the cross-entropy $H^{\dag}$ is defined by:
\begin{equation*}
\label{crossent-obj}
H^{\dag}(P(Y|X),Q(Y|X)) = - \sum_{i\in S_X} r_i \sum_{j \in S_Y} p_{i,j} \log q_{i,j}
\end{equation*}
(Note again the implicit dependence on $P(X)$).

Since $H(Y|X)$ is an irreducible constant with respect to the estimate $Q$, minimizing the KL-divergence  is equivalent to minimizing the cross-entropy.

However, in practice, we do not know the values $p_{i,j}$ or $r_i$. (If we knew them, there would be no need to estimate them).  This complicates our ability to both train and evaluate our models.  Our solution is to replace the ``true" distributions $P(X)$ and $P(Y|X)$ with their empirical counterparts $P^{e}(X)$ and $P^{e}(Y|X)$ (based on either the training set or a test set).  The hope is that if the sets are large enough, the empirical distributions serve as a reasonable proxy for the true distribution (at least with respect to the loss function).  Therefore, when training a model, we try to minimize $H^{\dag}(P^{e}(Y|X),Q(Y|X))$ (which also depends on $P^{e}(X)$).  This {\em empirical cross-entropy} can be re-written as a summation over the data points.
\begin{equation}
H^{\dag}(P^{\mbox{e}}(Y|X),Q(Y|X)) = - \frac{1}{n}\sum_{l=1}^n \log{q_{x_l, y_l}}
\end{equation}
This quantity is also known as the {\em log-loss} and is equivalent to the negative log-likelihood of the data (divided by the number of data points).

Now suppose that, rather than the standard KL-divergence, we wish to use its structured alternative.  Let $Z$ be a random partition taking values in $\mathbb{S} = \{\mathcal{S}_t\}$ and let $w_t = P(Z =\mathcal{S}_t)$. The analogous derivation yields the following loss function for the {\em structured empirical cross-entropy}.
\begin{equation*}
H_Z^{\dag}(P^{e}(Y|X),Q(Y|X)) = - \frac{1}{n}\sum_{l=1}^n \sum_{\mathcal{S}_t \in \mathbb{S}} w_t \log{q_{x_l, \mathcal{S}_i(y_l)}}
\end{equation*}
where
\begin{equation*}
 q_{i, \mathcal{S}(y)} = \sum_{j \in \mathcal{S}(Y)} q_{i,j}
\end{equation*}

In other words, we compute the standard empirical cross-entropy at varying levels of ``coarseness" and average the results.  This demonstrates, in theory, how we can enable models to use knowledge of the structure simply by adjusting the loss function.

\section{Defining Structure}
\label{sec:howtostructure}

Given a classification problem where the target variable has an a priori known structure, we propose using the {\em structured cross-entropy} as a loss function with an appropriate random partition.  This is a straightforward generalization of standard cross-entropy, using the structured entropy defined in the previous section.  Precise details on this derivation can be found in the Appendix.  The increase in computation time (to compute the standard entropy for each partition) is typically negligible relative to the entire training time.  All that is required is computing a cross-entropy value for each distinct partition.

One advantage to defining structure with partitions is its flexibility in capturing different kinds of structure.  In this section we describe three kinds of structure and for each, discuss how one might define a random partition.

\subsection{Hierarchical Structure}
\label{sec:hierstruct}
Frequently, the outcome variable naturally has a hierarchical structure akin to an evolutionary tree.  The outcome may have an explicitly evolutionary basis (e.g. species or genomic data),  but not necessarily so.  Diseases in the ICD-10 system, pharmaceutical drugs, and other man-made objects have a structure where the universe of possibilities is divided into large groups, each of which are subsequently refined for several levels.

Partitions work well in this scenario.  Imagine a rooted tree where the leaves represent the possible values of the random variable $Y$.  There is a natural mapping between internal nodes of the tree and subset of leaves that can be reached by walking away from the root.  Therefore, any minimal subset of node which separates the leaves from the root defines a partition.  In this way, it is easy to create a meaningful set of partitions and assign weights to them.

As an explicit example of this kind of taxonomy, consider the CIFAR100 image classification problem (\cite{krizhevsky2009learning}).  This collection of images is labeled with one of 100 classes, which in turn are intentionally grouped into 20 ``superclasses".   We extend the hierarchy further by defining 2 more levels: {\em category} and {\em supercategory}.  Category has 8 possible values: {\em animals-water, animals-mammals, animals-human, animals-other, objects-vehicle, objects-other, plants, landscape} and  supercategory contains just the 4 values {\em animals, plants, objects, landscape}.  

Let $Y$ represent the outcome variable of an image classification problem on CIFAR100.  Using this hierarchy, we can mathematically define a structure $\mathbb{S}$ of $Y$ in the following way.  We define the following partitions:
\begin{eqnarray*}
\mathcal{S}_0 &=& \mbox{the singleton partition on } S_Y \\
\mathcal{S}_1 &=& \mbox{the superclass partition on } S_Y\\
\mathcal{S}_2 &=& \mbox{the category partition on } S_Y \\
\mathcal{S}_3 &=& \mbox{the supercategory partition on } S_Y
\end{eqnarray*}
We can then define a random partition $Z$ which takes $\mathcal{S}_0,\mathcal{S}_1,\mathcal{S}_2,\mathcal{S}_3$ each with probability .25 to create the structured random variable $Y_Z$.

\subsection{Circular Structure}
\label{sec:circular-structure}
Imagine a problem where the goal is to predict the month of an event, given some data about the event (e.g. weather data, sales data, etc.).  Let $Y$ be our target variable, $S_Y = \{1,2,3,\ldots, 12\}$ where each number represents the corresponding month of the year.  The natural structure of the months is circular, not hierarchical.  Structured entropy
permits defining a structure $\mathbb{S}$ of $Y$ in the following way.  We define the following partitions:
\begin{eqnarray*}
\mathcal{S}_0 &=& \mbox{the singleton partition on } S_Y \\
\mathcal{S}_1 &=&  \{ \{1,2,3\},\{4,5,6\}, \{7,8,9\},\{10,11,12\} \} \\
\mathcal{S}_2 &=&  \{ \{2,3,4\},\{5,6,7\}, \{8,9,10\},\{11,12,1\} \} \\
\mathcal{S}_3 &=& \{ \{3,4,5\},\{6,7,8\}, \{9,10,11\},\{12,1,2\} \}  
\end{eqnarray*}
and then let $\mathbb{S}_B = \{\mathcal{S}_0, \mathcal{S}_1, \mathcal{S}_2, \mathcal{S}_3\}$.  We could then assign probability $p_0$ to $\mathcal{S}_0$ and probability $(1-p_0)/3$ to $\mathcal{S}_1, \mathcal{S}_2$, and $\mathcal{S}_3$.  In this case, we call the free parameter $p_0$ the {\em singleton probability} or {\em singleton weight}.

\subsection{Graphical Structure}
\label{sec:graphical-structure}
Many categorical variables possess a structure that can be represented by a graph, where the nodes represent the distinct values and edges represent ``neighbors".  The most obvious example is territories on a political map (e.g. nations, states, counties).  Thus, ``meaningful" partitions correspond to connected sets in the graph. As an example, consider the 48 contiguous U.S. states.  One could partition them according to defined regions, and simultaneously consider different definitions of region.  For example, Pennsylvania could be in the Northeast in one partition, but the Midwest in another.  Some partitions could be coarse (4 regions) and others could contain finer resolution (10 or more regions).  By enumerating a set of partitions and corresponding probabilities, we could incorporate our prior beliefs about the similarities of the states in question.

\subsubsection{Variable Random Partition}
\label{sec:var_rand_part}
Another approach is to have a variable, rather than fixed, loss function.  At each iteration, we could randomly choose one or more partitions which respect the structure of the graph (i.e. where each block in the partition is a connected set in the graph).  There are several ways to do this, we will discuss one particular method that we employ in Section~\ref{sec:experiments}.  This approach has two nuisance parameters: the singleton weight $p_0$ and the partition size $m$.
\begin{enumerate}
\item{Randomly choose a spanning tree (e.g. using Wilson's algorithm (\cite{wilson1996generating})).}
\item{Randomly remove a set of $(m-1)$ edges, creating a vertex partition $\mathcal{S}_1$ into $m$ connected sets.}
\item{Create a random partition where the singleton partition and $\mathcal{S}_1$ have probabilities $p_0$ and $(1-p_0)$.}
\item{Repeat the (random) creation of a random partition at each iteration, yielding a different structured cross-entropy loss function each time.}
\end{enumerate}

\begin{table}[t]
\caption{Experiment 1A Details} 
\label{expertable}
\begin{center}
\begin{tabular}{l|l|l|l}
\small{Train Size}  & \small{Trials}  & \small{LR} & \small{Epochs}\\
\hline 
50000 & 1  & .00005 & 600 \\
10000 & 5  & .00005 & 600 \\
5000 & 10 & .00005 & 1000 \\
2500 & 10 & .00003 & 1500 \\
1000 & 10 & .00003 & 1500
\end{tabular}
\end{center}
\end{table}

\section{Experimental Results}
\label{sec:experiments}
In this section we demonstrate the practical utility of structured entropy and loss functions derived from it.  We consider three prediction problems, each reflecting a different kind of structure in the target variable. For all experiments, comprehensive details can be found in the Appendix and code is available at \url{www.github.com/numeristical/resources/StrEntExp}.

\subsection{Problem 1: CIFAR100}
The first prediction problem involves image classification on CIFAR100.   We performed 4 experiments, each designed to test a particular image recognition scenario.  The detailed trajectories of the training runs can be found in the Appendix.

\subsubsection{Experiment 1A}
In Experiment 1A, we trained an ``off-the-shelf" CNN (\cite{keras_example}) using both the standard cross-entropy loss and a structured cross-entropy loss (using the structure defined in Section~\ref{sec:hierstruct}).  We used training sets of various sizes, and for each size ran multiple trials on non-overlapping training sets. The training set sizes, number of trials and learning rates used are shown in Table~\ref{expertable} and we used the RMSProp optimizer.  For all the experiments in this section, basic data augmentation techniques were utilized.  We recorded various metrics on the test set at the end of each epoch, and then averaged the trajectories across the trials.  The metrics recorded include the standard cross-entropy loss, structured cross-entropy loss, and the accuracy at each of the coarsened levels. 

The results are shown in Table~\ref{restable1A}. Here, (and in all the tables in this section)  the values given refer to the best test set results over the course of the (averaged) run.  The headings {\em Acc2}, {\em Acc3}, and {\em Acc4} refer to the accuracy at the superclass, category, and supercategory level, respectively, while {\em Cr Ent} refers to the standard cross-entropy loss.  We see that the model trained on the structured variant of cross-entropy loss outperforms the one trained on the standard entropy across all the considered metrics at all sample sizes.  

\begin{table}[t]
\caption{Experiment 1A Results} 
\label{restable1A}
\begin{center}
\begin{tabular}{|l|l|l|l|l|l|l|}
\hline 
\small{Tr Size} & \small{Loss} & \small{Cr Ent}  & \small{Acc} & \small{Acc2} & \small{Acc3} & \small{Acc4}\\
\hline 
1000 & Std & 4.099  & .1475 & .2643 & .3650 & .4954  \\
1000 & Struc & 3.999  & .1569 & .2873 & .3930 & .5171 \\
\hline 
2500 & Std & 3.521  & .2345 & .3613 & .4576 & .5648  \\
2500 & Struc & 3.429  & .2405 & .3829 & .4804 & .5833 \\
\hline 
5000 & Std & 2.999  & .2993 & .4335 & .5216 & .6149  \\
5000 & Struc & 2.931  & .3077 & .4578 & .5458 & .6334 \\
\hline 
10000 & Std & 2.543  & .3678 & .5076 & .5860 & .6665  \\
10000 & Struc & 2.485  & .3725 & .5337 & .6122 & .6821 \\
\hline 
50000 & Std & 2.197  & .4365 & .5843 & .6502 & .7135  \\
50000 & Struc & 2.186   & .4369 & .5952 & .6652 & .7189 \\
\hline 
\end{tabular}
\end{center}
\end{table}

Looking deeper, we see a couple of other trends.  First, the gains are more significant with the smaller datasets than with the larger ones.  This is to be expected: the less data, the more important it is to exploit prior information.  If you have 500 oaks and 500 maples in your training set, it may not hurt you to ignore that they are both kinds of trees, compared to having only 10 of each.  Second, the gains tend to be higher at the coarser levels than at the finer levels.  This also makes sense, since there is more aggregation of examples higher up in the hierarchy.

\subsubsection{Experiment 1B}
Experiment 1B was designed to determine if the improved performance arose from the coherent choice of hierarchy, or merely some arbitrary regularization.  We took the training sets of size 5,000 and trained them on a third variant of the loss function (referred to as the ``scrambled" variant).  In this variant, we used the structured entropy loss, but randomly permuted the classes.  Therefore, it had the same structural ``framework" as the structured cross-entropy, but the members in each group were randomly chosen and therefore unlikely to provide any coherence or similarity.  The same random permutation was used for all the trials.  

The results are shown in Figure~\ref{fig:exp1b}. We see that the ``scrambled" variant underperforms, demonstrating the improvement is truly dependent on the meaningful structure of the image classes.  

\begin{figure}[t]
\centering
    \begin{subfigure}[b]{0.25\textwidth}            
            \includegraphics[width=\textwidth]{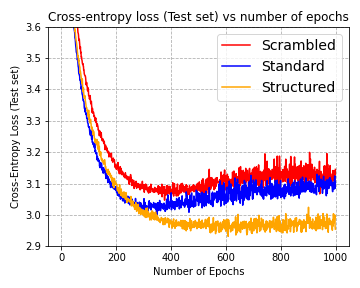}
            \caption{Cross-Entropy Loss}
            \label{fig:exp1bloss}
    \end{subfigure}%
    \begin{subfigure}[b]{0.25\textwidth}
            \centering
            \includegraphics[width=\textwidth]{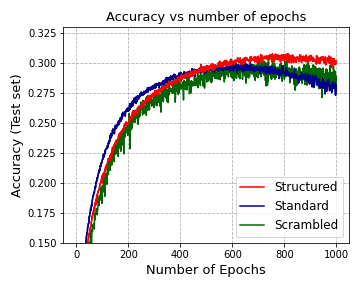}
            \caption{Accuracy}
            \label{figexp1bacc}
    \end{subfigure}
   \caption{Experiment 1B Results (5K training set)}\label{fig:exp1b}
\end{figure}

\subsubsection{Experiment 1C}
Experiment 1C was designed to address whether the gains seen in Experiment 1A would occur on modern CNN architectures with pre-trained weights.  We used ResNet50 (~\cite{he2016deep}), and MobileNetV2 (~\cite{sandler2018mobilenetv2}) initialized with pre-trained weights from ImageNet (~\cite{imagenet-cvpr09}).  This time we used the entire 50K image training set, and the Adam optimizer with learning rate of 0.00001.  The entire network was trained, with the intuition that even the primitive features learned in the early layers could benefit from understanding the structure.  The results, found in Table~\ref{restable1C} largely echo what was found in Experiment 1A.  The models trained on structured cross-entropy outperformed the standard counterpart on all metrics.  Exploiting the structure in the training images adds value despite the fact that these neural networks were pre-trained on a much larger dataset.

\begin{table}[b]
\caption{Experiment 1C Results} 
\label{restable1C}
\begin{center}
\begin{tabular}{|l|l|l|l|l|l|l|}
\hline 
\small{Model} & \small{Loss} & \small{Cr Ent}  & \small{Acc} & \small{Acc2} & \small{Acc3} & \small{Acc4}\\
\hline 
RN50 & Std & .7956  & .8028 & .8857 & .9085 & .9233  \\
RN50 & Str & .7772  & .8075 & .8926 & .9147 & .9276 \\
\hline 
MN2 & Std & .7260  & .8051 & .8894 & .9109 & .9264  \\
MN2 & Str & .7224  & .8080 & .8943 & .9166 & .9313 \\
\hline 
\end{tabular}
\end{center}
\end{table}

\begin{table}[b]
\caption{Experiment 1D Results} 
\label{restable1D}
\begin{center}
\begin{tabular}{|l|l|l|l|l|l|}
\hline 
 \small{Loss} & \small{Cr Ent}  & \small{Acc} & \small{Acc2} & \small{Acc3} & \small{Acc4}\\
\hline 
 Std & 2.556  & .3620 & .5300 & .6132 & .6772  \\
 Str & 2.528  & .3658 & .5414 & .6260 & .6878 \\
\hline 
\end{tabular}
\end{center}
\end{table}

\subsubsection{Experiment 1D}
Experiment 1D was designed to see if the modified loss function would help in a transfer learning problem, where only the last layer of the network is trained.  We used the MobileNetV2 architecture with weights from ImageNet trained on  small training sets: just 1,000 images.  We repeated 3 trials and averaged the results, using the SGD optimizer with learning rate .002 for 300 epochs.

The results are shown in Table~\ref{restable1D}.  We see the same general pattern as in the previous experiments: a small gain in accuracy at the class level, with more significant gains at the coarser levels, and a lower cross-entropy.   This demonstrates that structured entropy can add value in a transfer learning scenario, which often involves small training data sets.

\subsection{Problem 2: Month Prediction}
To demonstrate the flexibility of the structured cross-entropy loss function, we ran experiments on a weather data set collected from NOAA (\url{https://www.ncdc.noaa.gov/cdo-web/}).  Each row of the dataset contains weather observations on a particular day from a weather station in California.  We consider the problem of trying to predict the month of the year given information about the weather conditions (temperature and precipitation).

We implemented gradient boosting with vector-valued leafs and diagonal Hessians as discussed in \cite{CompactMCTrees}.  For various training set sizes, we compare using the standard cross-entropy loss to two structured variants: a structured cross-entropy loss using the fixed structure defined in Section~\ref{sec:circular-structure}, and a variable structured cross-entropy loss as described in Section~\ref{sec:var_rand_part} with a partition size of 4.  We vary the singleton weight $p_0$ from 0.1 to 0.9 and consider the log loss on an independent test set.  The results (averaged over 5 trials, and optimized over max-depth values) are in Figure~\ref{fig:expgb2a}.  We see that the fixed structure outperforms standard cross-entropy for all choices of $p_0$, with the best performance achieved with $p_0$ around 0.2 or 0.3.  Moreover, the random structure outperforms standard cross entropy for all but the smallest values of $p_0$, doing best around $p_0 =0.5$.

\begin{figure}[t]
\centering
\includegraphics[scale=0.32]{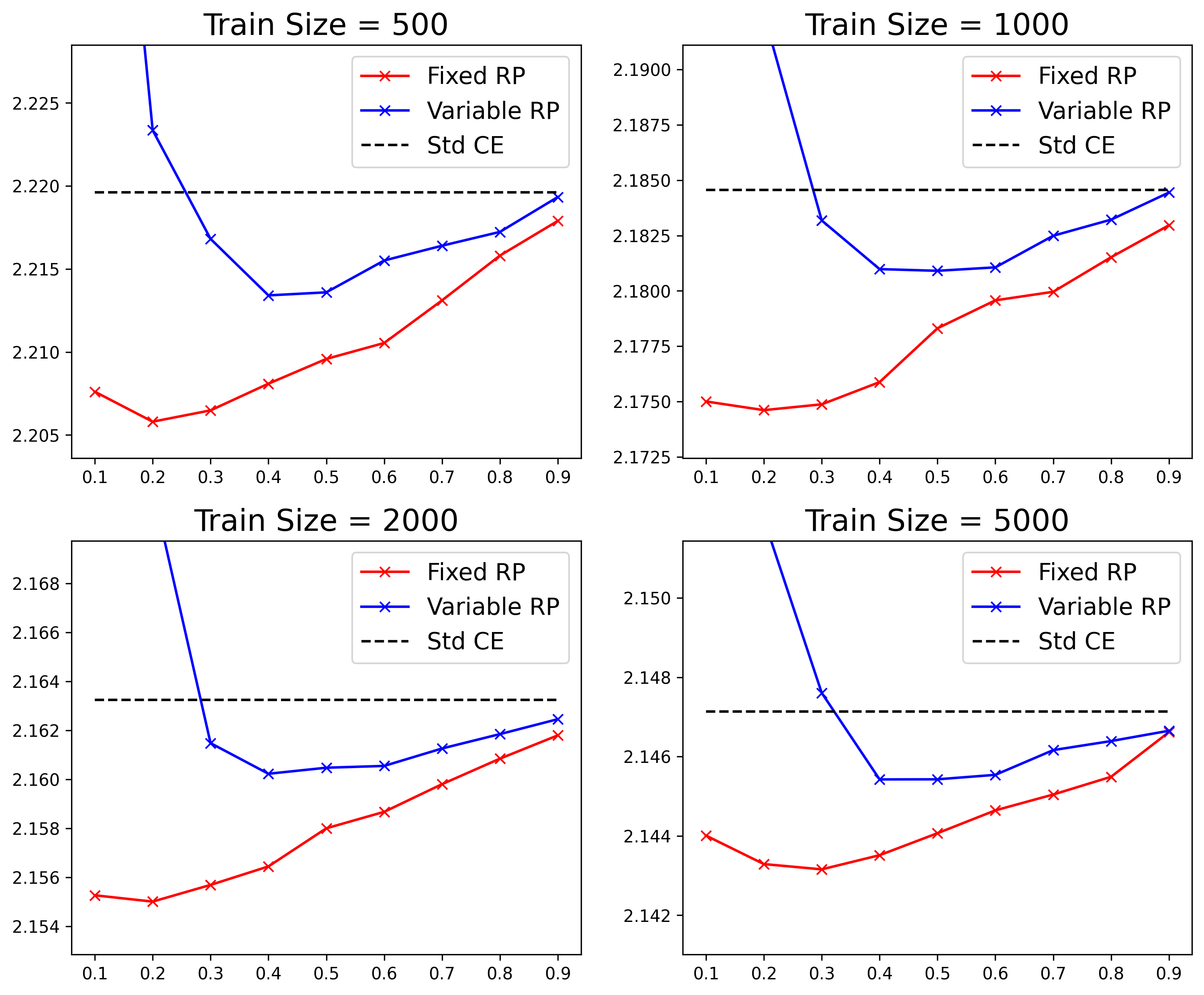}
\caption{Log Loss vs Singleton Weight}
\label{fig:expgb2a}
\end{figure}

\subsection{Problem 3: County Prediction}
To show the utility of this approach on more complicated structures, we consider the same dataset as the previous experiment, except that we try to predict the county of the observation, rather than the month.  The counties of California have a natural structure given by geographical adjacency which can be represented by a graph.  Since creating fixed random partitions for this problem could be tedious (though by no means impossible), we consider just the variable random partition approach for this problem, with 5 different combinations of partition size and singleton weight.  The results (averaged over 5 trials, and optimized over max-depth values) are in Table~\ref{grid-table2}.  We see that all 5 combinations of nuisance parameters outperformed standard cross-entropy.  Generally, choosing a random partition size of 10 and a singleton weight of 0.5 did best.  This shows that the variable random partition approach is feasible, valuable, and not overly dependent on specific values of the nuisance parameters.

\section{Summary}
We defined a variant of entropy, called {\em structured entropy} which can account for the similarity of the different values of the variable in question.  Our approach defines this similarity by means of a {\em random partition}, effectively putting weights on different coarsenings of the state space.  Due to the simplicity of the construction, we are easily able to prove it retains many of the desirable properties of standard Shannon entropy.  The tight link between entropy, cross-entropy, and KL-Divergence motivates the use of structured cross-entropy as a loss function for classification problems where there is known a-priori structure in the target variable.  Moreover, the approach is applicable to a wide variety of types of structure, as opposed to many existing techniques that are restricted to hierarchical structure.  Finally, as opposed to methods based on optimal transport, it is feasible for use in methods like neural networks and gradient boosting, which require the computation of the loss function at each iteration.

We demonstrate the effectiveness of the approach on several different classification problems with different kinds of structure in the target variable.  Our experiments on CIFAR-100 show that it can be valuable when training models from scratch, on large models with pre-trained weights, and in a transfer learning scenario with small datasets.  We also show it can be used for gradient boosting, and on circular and graphical structures that are not hierarchical in nature.  For graph-based structures, we show that using a variable random partition can prove to be a viable alternative in cases where it is undesirable to manually create a single random partition.

This work suggests several areas for future exploration.  One question that arises is whether it would be possible to learn the structure (e.g. from a confusion matrix) and then exploit it.  In a similar vein, this may provide an avenue for creating a ``hypothesis test" for structure: if adding a priori structural knowledge improves performance, then it suggests the structure is meaningful and significant.  These ideas may be of interest to the structure learning community.  More directly, there may be classification problems in other fields for which defining structure via partitions is a tractable approach and leads to better predictive performance.

\begin{table}[t]
\caption{Log Loss values for County Prediction} \label{grid-table2}
\begin{center}
\begin{tabular}{|r|c|c|c|c|c|c|}
\hline
 \multicolumn{2}{|c|}{Part. Size} & 10 & 10 & 10 & 5 & 20 \\
\hline
\multicolumn{2}{|c|}{Singleton Wt} &0.25&0.5&0.75 & 0.5 & 0.5 \\
\hline
Train &Stnd & \multicolumn{5}{c|}{} \\
Size& Cr Ent & \multicolumn{5}{c|}{}\\
\hline 
500       & 3.883 &  \textbf{3.865} &3.866& 3.876  & 3.868 & 3.869\\
\hline 
1000         & 3.650 & 3.639 & \textbf{3.634} & 3.643 & 3.636 & 3.637\\
\hline 
2000         & 3.534 & 3.532 & \textbf{3.521} & 3.529 & 3.523 & 3.522\\
\hline 
5000        &3.473 & 3.466 & \textbf{3.462}  & 3.468 & 3.465 & 3.463\\
\hline 
10K         & 3.441 & 3.440 & \textbf{3.434} & 3.438 & 3.435 & 3.434\\
\hline 
20K         & 3.417 & 3.416 & \textbf{3.412} & 3.415 & 3.413 & \textbf{3.412}\\
\hline
\end{tabular}
\end{center}
\end{table}

\bibliography{str_ent_arxiv}

\newpage
\appendix
\onecolumn
\section{Appendix}
This Appendix contains 3 parts:
\begin{itemize}
\item{Proofs of Theorems in Section 2}
\item{Additional Experimental Details}
\item{CIFAR100 hierarchy definition}
\end{itemize}

\subsection{Proofs of Theorems in Section 2}
\begin{theorem}
Let $Y$ be a structured random variable with a random partition $Z$ defined on a structure $\mathbb{S}$ over $S(Y)$.  Let $X$ be a structured random variable with a random partition $W$ defined on a structure $\mathbb{R}$ over $S(X)$.  
Let $q_i = P(Z=\mathcal{S}_i)$ and let $r_i = P(W=\mathcal{R}_i)$.
  Then $H(Y_Z|X_W) \leq H(Y_Z)$
\end{theorem}
\begin{proof}
This follows directly by handling each combination of partitions individually and using the corresponding law from standard entropy.  \begin{eqnarray*}
H(Y_Z|X_W) &=& \sum_{\mathcal{S}_i \in \mathbb{S}} \sum_{\mathcal{R}_j\in \mathbb{R}} q_i r_j H(\mathcal{S}_i(Y)|\mathcal{R}_j(X)) \\
		    &\leq& \sum_{\mathcal{S}_i \in \mathbb{S}} \sum_{\mathcal{R}_j\in \mathbb{R}} q_i r_j H(\mathcal{S}_i(Y)) \\
		    &=& \sum_{\mathcal{S}_i \in \mathbb{S}}  q_i  H(\mathcal{S}_i(Y)) \left(\sum_{\mathcal{R}_j\in \mathbb{R}}r_j\right)\\
		    &=& H(Y_Z)
\end{eqnarray*}
\end{proof}

The next two theorems assume the same setup as Theorem 2.1.

\begin{theorem}
The structured mutual information is symmetric.
\begin{equation}
I(Y_Z;X_W) = I(X_W;Y_Z)
\end{equation}
\end{theorem}
\begin{proof}
Since $\sum_{\mathcal{R}_j \in \mathbb{R}} r_j  = 1$, we can rewrite $H(Y_Z)$ in the following way:
\begin{eqnarray*}
H(Y_Z) &=&\sum_{\mathcal{S}_i \in \mathbb{S}} q_i H(\mathcal{S}_i(Y)) \left(\sum_{\mathcal{R}_j \in \mathbb{R}} r_j \right) \\
H(Y_Z) &=& \sum_{\mathcal{S}_i \in \mathbb{S}} \sum_{\mathcal{R}_j \in \mathbb{R}}q_i  r_j H(\mathcal{S}_i(Y))  
\end{eqnarray*}
and therefore:
\begin{equation*}
I(Y_Z;X_W)= \sum_{\mathcal{S}_i \in \mathbb{S}} \sum_{\mathcal{R}_j \in \mathbb{R}}q_i  r_j (H(\mathcal{S}_i(Y)  - H(\mathcal{S}_i(Y)|\mathcal{R}_j(X)) \\
\end{equation*} 
Since $\mathcal{S}_i(Y)$ and $\mathcal{R}_j(X)$ are standard random variables, we can apply the symmetry of their mutual information
\begin{equation*}
I(Y_Z;X_W)= \sum_{\mathcal{S}_i \in \mathbb{S}} \sum_{\mathcal{R}_j \in \mathbb{R}}q_i  r_j (H(\mathcal{R}_i(X)  - H(\mathcal{R}_i(X)|\mathcal{S}_j(Y)) \\
\end{equation*} 
and therefore
\begin{equation*}
I(Y_Z;X_W)= H(X_W) -H(X_W|Y_Z)  = I(X_W;Y_Z)\\
\end{equation*} 
as desired.
\end{proof}

\begin{theorem}
$H(X_W, Y_Z) = H(X_W) + H(Y_Z|X_W)$
\end{theorem}

\begin{proof}
The proof is straightforward from the definitions.
\begin{eqnarray*}
H(X_W, Y_Z) &=& H((X,Y)_{W\times Z} \\
		      &=& \sum_{(\mathcal{R}_i , \mathcal{S}_j )\in \mathbb{R} \times \mathbb{S}}  r_i q_j H((\mathcal{R}_i , \mathcal{S}_j )(X,Y)) \\
		      &=& \sum_{\mathcal{R}_i \in \mathbb{R}}\sum_{\mathcal{S}_j \in \mathbb{S}}  r_i q_j H((\mathcal{R}_i(X), \mathcal{S}_j (Y))) \\
		      &=&  \sum_{\mathcal{R}_i \in \mathbb{R}}\sum_{\mathcal{S}_j \in \mathbb{S}}  r_i q_j H(\mathcal{R}_i(X)) \\
		      & & + \sum_{\mathcal{R}_i \in \mathbb{R}}\sum_{\mathcal{S}_j \in \mathbb{S}}  r_i q_j H(\mathcal{S}_j (Y))|\mathcal{R}_i(X)) \\
		      &=& H(X_W) + H(Y_Z|X_W)
\end{eqnarray*}
\end{proof}

\vfill
\pagebreak

\subsection{Additional Experimental Details}
In this section, we show the trajectories of (standard) cross-entropy loss, accuracy, and the coarsened variants of accuracy.

\subsubsection{Experiment 1A Details}
For the most part, we see a clear separation between the trajectories of the runs using standard cross entropy and structured cross-entropy.  This demonstrates that the numerical results in the paper were not a coincidence of a single favorable epoch.  The trajectories on the larger data sets show more spikiness for two reasons.  First, an epoch represents more updates on a larger training set than on a smaller one.  Second, we are averaging over more trials on the smaller data sets (since we can create more non-overlapping training sets).  For the 50K dataset there is only one trial, as only one dataset is possible.

\begin{figure}[h]
\centering
    \begin{subfigure}[h]{0.33\textwidth}            
            \includegraphics[width=\textwidth]{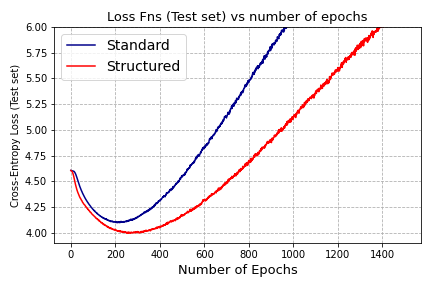}
            \caption{Cross-Entropy Loss (1K)}
            \label{fig:exp2loss}
    \end{subfigure}%
    \begin{subfigure}[h]{0.33\textwidth}
            \centering
            \includegraphics[width=\textwidth]{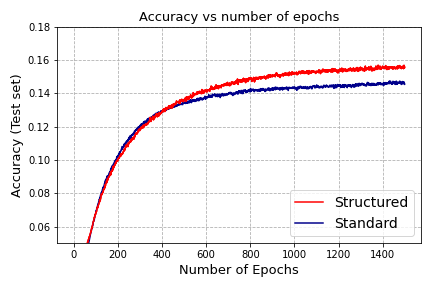}
            \caption{Accuracy (1K)}
            \label{figexp2acc}
    \end{subfigure}
    \begin{subfigure}[h]{0.33\textwidth}
            \centering
            \includegraphics[width=\textwidth]{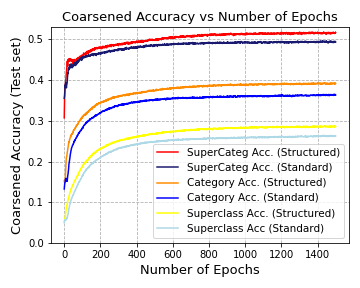}
            \caption{Coarsened Accuracy (1K)}
            \label{figexp2acc}
    \end{subfigure}
   \caption{Results (size 1K training set)}\label{fig:exp4ab}
\end{figure}

\begin{figure}[h]
\centering
    \begin{subfigure}[h]{0.33\textwidth}            
            \includegraphics[width=\textwidth]{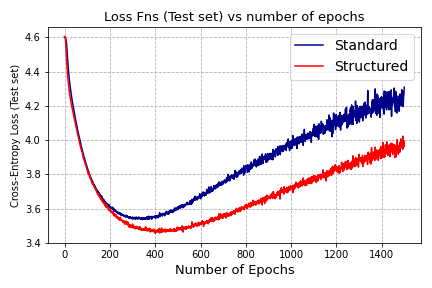}
            \caption{Cross-Entropy Loss (2.5K)}
            \label{fig:exp2loss}
    \end{subfigure}%
    \begin{subfigure}[h]{0.33\textwidth}
            \centering
            \includegraphics[width=\textwidth]{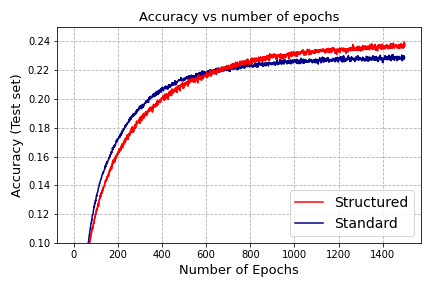}
            \caption{Accuracy (2.5K)}
            \label{figexp2acc}
    \end{subfigure}
    \begin{subfigure}[h]{0.33\textwidth}
            \centering
            \includegraphics[width=\textwidth]{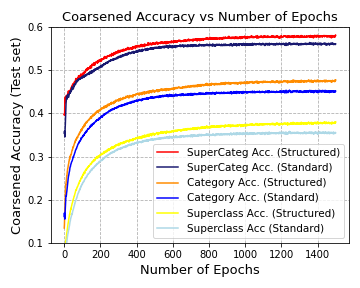}
            \caption{Coarsened Accuracy (2.5K)}
            \label{figexp2acc}
    \end{subfigure}
   \caption{Results (size 2,500 training set)}\label{fig:exp4ab}
\end{figure}

\begin{figure}[h!]
\centering
    \begin{subfigure}[h]{0.33\textwidth}            
            \includegraphics[width=\textwidth]{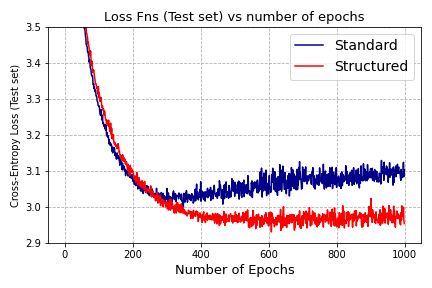}
            \caption{Cross-Entropy Loss (5K)}
            \label{fig:exp2loss}
    \end{subfigure}%
    \begin{subfigure}[h]{0.33\textwidth}
            \centering
            \includegraphics[width=\textwidth]{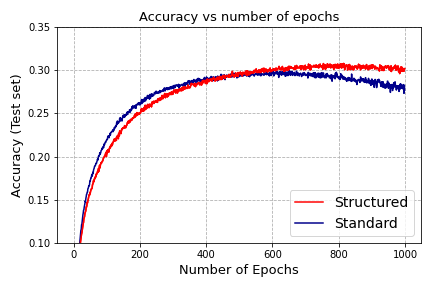}
            \caption{Accuracy (5K)}
            \label{figexp2acc}
    \end{subfigure}
    \begin{subfigure}[h]{0.33\textwidth}
            \centering
            \includegraphics[width=\textwidth]{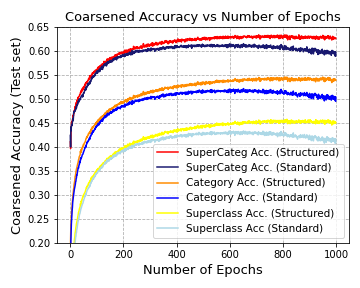}
            \caption{Coarsened Accuracy (5K)}
            \label{figexp2acc}
    \end{subfigure}
   \caption{Results (size 5K training set)}\label{fig:exp2ab}
\end{figure}

\begin{figure}[h!]
\centering
    \begin{subfigure}[h]{0.33\textwidth}            
            \includegraphics[width=\textwidth]{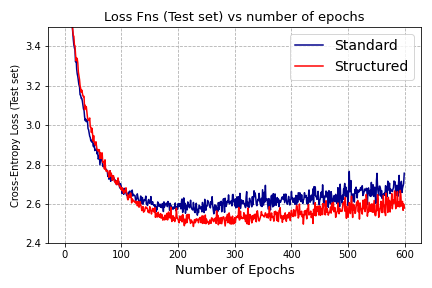}
            \caption{Cross-Entropy Loss (10K)}
            \label{fig:exp2loss}
    \end{subfigure}%
    \begin{subfigure}[h]{0.33\textwidth}
            \centering
            \includegraphics[width=\textwidth]{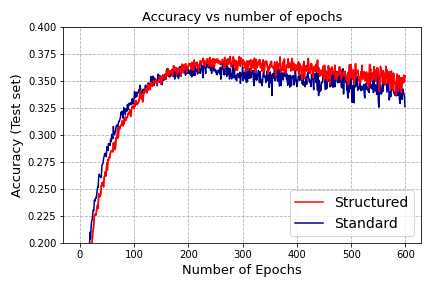}
            \caption{Accuracy (10K)}
            \label{figexp2acc}
    \end{subfigure}
    \begin{subfigure}[h]{0.33\textwidth}
            \centering
            \includegraphics[width=\textwidth]{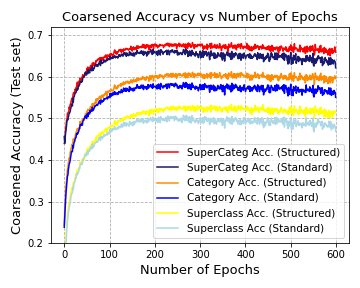}
            \caption{Coarsened Accuracy (10K)}
            \label{figexp2acc}
    \end{subfigure}
   \caption{Results (size 10K training set)}\label{fig:exp5ab}
\end{figure}

\begin{figure}[h!]
\centering
    \begin{subfigure}[b]{0.33\textwidth}            
            \includegraphics[width=\textwidth]{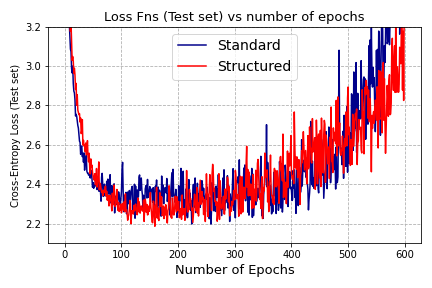}
            \caption{Cross-Entropy Loss (50K)}
            \label{50Kloss}
    \end{subfigure}%
    \begin{subfigure}[b]{0.33\textwidth}
            \centering
            \includegraphics[width=\textwidth]{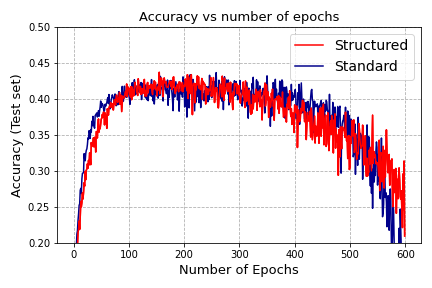}
            \caption{Accuracy (50K)}
            \label{figexp2acc}
    \end{subfigure}
    \begin{subfigure}[b]{0.33\textwidth}
            \centering
            \includegraphics[width=\textwidth]{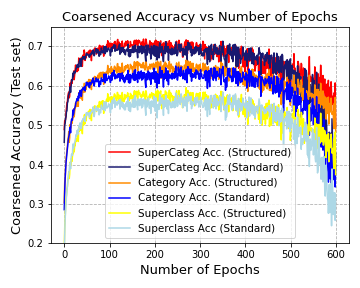}
            \caption{Coarsened Accuracy (50K)}
            \label{figexp2acc}
    \end{subfigure}
   \caption{Results (size 50K training set)}\label{simple50K}
\end{figure}
\vspace{4in}

For the 50K training set the separation is less clear, but still the cross-entropy and coarsened accuracies shows a better trajectory, especially between epochs 100 and 200 where performance is best.

\vspace{3in}
\vfill
\pagebreak[4]
\subsubsection{Trajectories for Experiment 1C}
Here we display the full trajectories for Experiment 1C: for ResNet50 and MobileNetV2.  In both cases we trained on the entire 50K training set.  Again, we generally see clear separation between the trajectories.

\begin{figure}[h!]
\centering
    \begin{subfigure}[h]{0.33\textwidth}            
            \includegraphics[width=\textwidth]{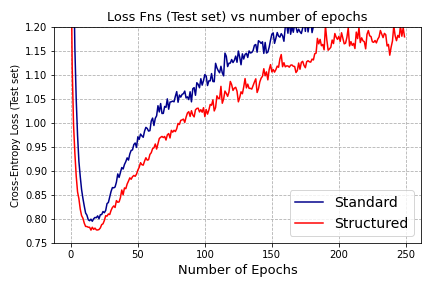}
            \caption{Cross-Entropy Loss}
            \label{fig:exp2loss}
    \end{subfigure}%
    \begin{subfigure}[h]{0.33\textwidth}
            \centering
            \includegraphics[width=\textwidth]{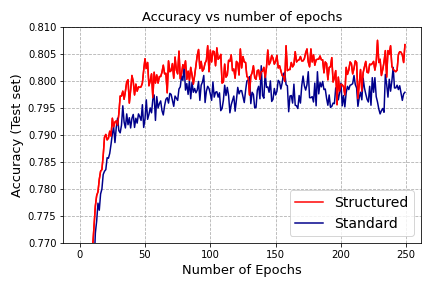}
            \caption{Accuracy}
            \label{figexp2acc}
    \end{subfigure}
    \begin{subfigure}[h]{0.33\textwidth}
            \centering
            \includegraphics[width=\textwidth]{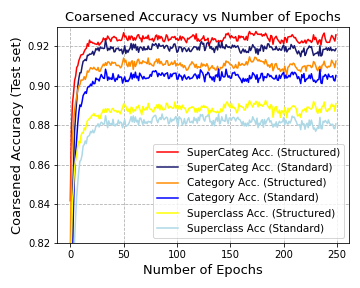}
            \caption{Coarsened Accuracy}
            \label{figexp2acc}
    \end{subfigure}
   \caption{Results (ResNet50 50K)}\label{fig:exp4ab}
\end{figure}

\begin{figure}[h!]
\centering
    \begin{subfigure}[h]{0.33\textwidth}            
            \includegraphics[width=\textwidth]{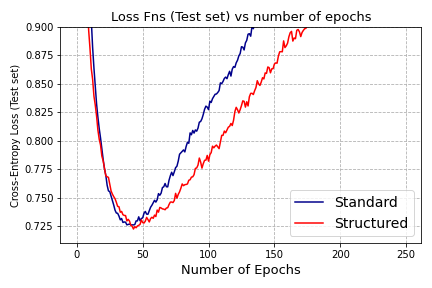}
            \caption{Cross-Entropy Loss}
            \label{fig:exp2loss}
    \end{subfigure}%
    \begin{subfigure}[h]{0.33\textwidth}
            \centering
            \includegraphics[width=\textwidth]{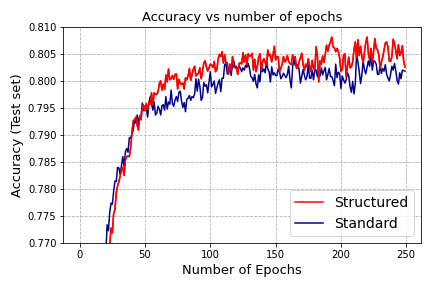}
            \caption{Accuracy}
            \label{figexp2acc}
    \end{subfigure}
    \begin{subfigure}[h]{0.33\textwidth}
            \centering
            \includegraphics[width=\textwidth]{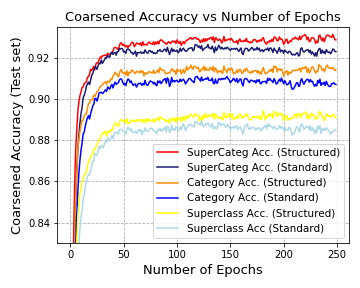}
            \caption{Coarsened Accuracy}
            \label{figexp2acc}
    \end{subfigure}
   \caption{Results (MobileNetV2 50K)}\label{fig:exp4ab}
\end{figure}

\vfill
\pagebreak

\subsection{Trajectories for Experiment 1D}
Here we display the full trajectories for Experiment 1D, which only re-trained the top layer.

\begin{figure}[h!]
\centering
    \begin{subfigure}[b]{0.35\textwidth}            
            \includegraphics[width=\textwidth]{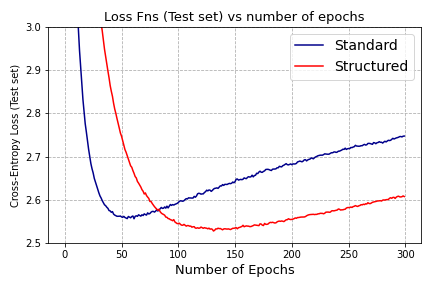}
            \caption{Losses}
            \label{fig:exp3loss}
    \end{subfigure}%
    \begin{subfigure}[b]{0.25\textwidth}
            \centering
            \includegraphics[width=\textwidth]{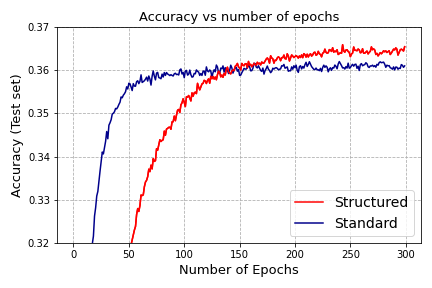}
            \caption{Accuracy}
            \label{fig:exp3acc}
    \end{subfigure}
    \begin{subfigure}[b]{0.35\textwidth}
            \centering
            \includegraphics[width=\textwidth]{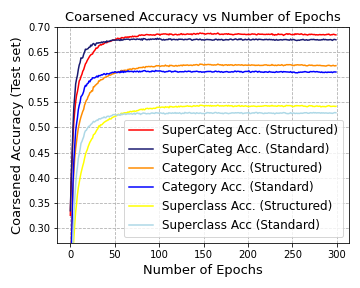}
            \caption{Accuracy}
            \label{fig:exp3coarse}
    \end{subfigure}
   \caption{Experiment 1D Results }\label{fig:exp3ab}
\end{figure}

\vfill
\pagebreak[4]
\subsection{CIFAR100 HIERARCHY}
Here we give details of how the 100 classes mapped to superclass, category and subcategory.  The mapping in Table~\ref{cl-scl-table} (Class to Superclass) is included by the creators of CIFAR100.  The other two mappings were developed by the authors.
\begin{table}[h!]
\caption{Map from Class to Superclass} 
\label{cl-scl-table}
\begin{center}
\begin{tabular}{ll}
\textbf{Class}  &\textbf{Superclass} \\
\hline \\
beaver, dolphin, otter, seal, whale & aquatic mammals \\
aquarium fish, flatfish, ray, shark, trout & fish \\
orchids, poppies, roses, sunflowers, tulips & flowers \\
bottles, bowls, cans, cups, plates & food containers \\
apples, mushrooms, oranges, pears, sweet peppers & fruit and vegetables \\
clock, computer keyboard, lamp, telephone, television & household electrical devices \\
bed, chair, couch, table, wardrobe & household furniture \\
bee, beetle, butterfly, caterpillar, cockroach & insects \\
bear, leopard, lion, tiger, wolf & large carnivores \\
bridge, castle, house, road, skyscraper & large man-made outdoor thing \\
cloud, forest, mountain, plain, sea & large natural outdoor scenes \\
camel, cattle, chimpanzee, elephant, kangaroo & large omnivores and herbivore \\
fox, porcupine, possum, raccoon, skunk & medium-sized mammals \\
crab, lobster, snail, spider, worm & non-insect invertebrates \\
baby, boy, girl, man, woman & people \\
crocodile, dinosaur, lizard, snake, turtle & reptiles \\
hamster, mouse, rabbit, shrew, squirrel & small mammals \\
maple, oak, palm, pine, willow & trees \\
bicycle, bus, motorcycle, pickup truck, train & vehicles 1 \\
lawn-mower, rocket, streetcar, tank, tractor & vehicles 2 \\
\end{tabular}
\end{center}
\end{table}

\begin{table}[t]
\caption{Map from Superclass to Category } 
\label{sc-cat-table}
\begin{center}
\begin{tabular}{ll}
\textbf{Superclass}  &\textbf{Category} \\
\hline \\
aquatic mammals & animals-water\\
fish & animals-water\\
flowers & plants\\
food containers & objects-other\\
fruit and vegetables & plants\\
household electrical devices & objects-other\\
household furniture & objects-other\\
insects & animals-other\\
large carnivores & animals-mammals \\
large man-made outdoor thing & landscape \\
large natural outdoor scenes  & landscape \\
large omnivores and herbivore &animals-mammals \\
medium-sized mammals & animals-mammals\\
non-insect invertebrates & animals-other \\
people & animals-human\\
reptiles & animals-other\\
small mammals & animals-mammals\\
trees & plants\\
vehicles 1 & objects-vehicle\\
vehicles 2 & objects-vehicle\\
\end{tabular}
\end{center}
\end{table}

\begin{table}[t]
\caption{Map from Category to Supercategory } 
\label{cat-scat-table}
\begin{center}
\begin{tabular}{ll}
\textbf{Category}  &\textbf{Supercategory} \\
\hline \\
animals-water & animals\\
animals-mammals & animals\\
animals-human & animals\\
animals-other & animals\\
plants & plants\\
objects-vehicle & objects\\
objects-other & objects\\
landscape & landscape \\
\end{tabular}
\end{center}
\end{table}

\end{document}